\documentclass[11pt]{article}

\usepackage[final]{acl}

\usepackage{times}
\usepackage{latexsym}
\usepackage{amsmath,amsthm}
\usepackage{amsfonts}
\usepackage{enumitem}
\usepackage{array}
\usepackage{caption}
\usepackage{makecell}
\usepackage[utf8]{inputenc}
\usepackage{geometry}
\usepackage{booktabs}
\usepackage{multirow}
\usepackage{graphicx}
\usepackage[table,xcdraw]{xcolor}
\usepackage{amssymb}
\usepackage[T1]{fontenc}

\usepackage[utf8]{inputenc}

\usepackage{microtype}

\usepackage{inconsolata}

\usepackage{graphicx}

\usepackage{titlesec}
\titlespacing*{\section}{0pt}{0.1\baselineskip}{0.1\baselineskip}
\titlespacing*{\subsection}{0pt}{0.1\baselineskip}{0.1\baselineskip}
\titlespacing*{\subsubsection}{0pt}{0.1\baselineskip}{0.1\baselineskip}

\usepackage{tcolorbox}   
\definecolor{myblue}{RGB}{230, 245, 255}
\definecolor{mygray}{RGB}{240, 240, 240}

\theoremstyle{plain}
\newtheorem{theorem}{Theorem}[section]

\newcommand{\hl}{\textcolor{blue}}

\newcommand{\taghl}{\colorbox{gray!35}}

\title{Decompose, Look, and Reason: Reinforced Latent Reasoning for VLMs}

\author{
    \textbf{Mengdan Zhu}$^{\clubsuit}$,
    \textbf{Senhao Cheng}$^{\spadesuit}$,
    \textbf{Liang Zhao}$^{{\clubsuit}}$\thanks{Corresponding author.}  \\
    $^{\clubsuit}$Emory University \quad
    $^{\spadesuit}$University of Michigan, Ann Arbor \\
    \texttt{\{mengdan.zhu, liang.zhao\}@emory.edu} \\
    \texttt{senhaoc@umich.edu}
}

\begin{document}
\maketitle
\begin{abstract}
Vision-Language Models often struggle with complex visual reasoning due to the visual information loss in textual CoT. Existing methods either add the cost of tool calls or rely on localized patch-based embeddings that are insufficient to extract semantics in multi-step reasoning. We propose \emph{"Decompose, Look, and Reason" (DLR)}, a reinforced latent reasoning framework that dynamically decomposes queries into textual premises, extracts premise-conditioned continuous visual latents, and deduces answers through grounded rationales. We introduce a three-stage training pipeline and propose a novel \emph{Spherical Gaussian Latent Policy}, to enable effective exploration in the latent space. Extensive experiments on vision-centric benchmarks show that DLR consistently outperforms strong baselines, including text-only, interleaved multimodal CoT, and latent reasoning methods, while providing superior stepwise interpretability.
\end{abstract}

\section{Introduction}
Vision-Language Models (VLMs) have made remarkable progress on visual question answering~\cite{goyal2017making,antol2015vqa}, vision-language understanding~\cite{li2023blip,zhu2025latentexplainer}, vision-language generation~\cite{zhang2021ernie,zhu2025cross}, and vision-language reasoning~\cite{wang2024exploring,yu2026latent}. The success of integrating Chain-of-Thought (CoT) into large language models (LLMs) for enhancing reasoning on complex tasks has inspired its extension to multimodal reasoning. Early approaches relied on text-only multimodal CoT (MCoT)~\cite{zhang2023multimodal}, which translates visual inputs into textual descriptions. However, this textual abstraction inevitably loses some important visual information. Therefore, subsequent work evolved toward interleaved MCoT, which explicitly incorporates localized visual signals like cropped patches or bounding boxes~\cite{shao2024visual,gao2025interleaved} into the reasoning steps. This line of work later developed into "thinking with images", where the model actively manipulates or edits images during reasoning. Such edits mainly include zooming in~\cite{shen2025zoomeye,zhang2025chain,su2026pixel}, drawing auxiliary lines and sketches~\cite{hu2024visual}, and highlighting boxing regions for understanding~\cite{fu2025refocus}. Despite being visually grounded, these approaches incur extra cost due to external tool calls or program-based image manipulations, and are inherently limited by the set of available external tools. Alternatively, latent space reasoning models can be more efficient, as they avoid external tool calling by directly projecting intermediate visual information into an internally continuous embedding space. Existing latent visual reasoning methods typically learn continuous representations by reconstructing query-relevant region-of-interest (ROI) visual embeddings~\cite{li2026latent} or the embeddings of auxiliary images~\cite{wang2025monet}. 

Patch-based methods such as Interleaved MCoT and "thinking with images" approaches, and prior latent reasoning methods have the following limitations: Such methods remain tied to explicit localized visual regions but fail to semantically isolate the desired elements. (1) An ROI or patch-based method may \emph{over-include} information by bundling together all visual content within a selected region, including irrelevant context that are not needed for the current reasoning step. (2) At the same time, it may \emph{under-include} when the required evidence is inherently non-local, such as a global layout, more abstract concepts, or a cross-patch relation. For example, Fig.~\ref{fig:qualitative}(b) requires global and cross-patch visual information to determine the dominant color. Moreover, prior latent visual reasoning methods often insert the latent only once. For instance, Fig.~\ref{fig:qualitative}(a) presents a complex logical reasoning problem, where each reasoning step requires attending to different regions and their logical relationship, which cannot be captured by a single ROI.

To address these limitations, we introduce \textbf{\underline{D}ecompose, \underline{L}ook, and \underline{R}eason: Reinforced Latent Reasoning (DLR)} framework for VLMs. Our framework mimics the human cognitive process of "Decompose $\rightarrow$ Look $\rightarrow$ Reason": (1) \textbf{Decompose the premise dynamically}: The model dynamically generates a textual premise or subquestion, determining \textit{what} specific details need to be verified in the image and \textit{when} to look for them. (2) \textbf{Look premise-conditioned latents}: A visual grounder attends to the image and conditioned on the hidden state of the textual premise, extracting continuous latent embeddings that capture \textit{where} to look. Unlike patch-based latent reasoning methods, these latent tokens enable more efficient representation of image-related visual thoughts, covering both localized visual information and non-local latent semantics. (3) \textbf{Reason the latent-grounded rationale}: Conditioned on the injected visual latents, the VLM generates a textual rationale to explain in detail and eventually deduce the final answer. 

To unleash the potential of this DLR framework, we propose a progressive three-stage training pipeline. In stage I: Pretraining, we establish a foundational cross-modal alignment and ensure that the continuous latents can accurately extract visual semantics corresponding to specific textual conditions. Next, in Stage II: Supervised Finetuning, we teach the VLM with the decomposition capability to internalize the structured DLR format. However, SFT strictly relies on teacher-forced log-likelihood, inherently bounding the visual grounder's capacity to deterministic feature extraction without active exploration. To break this bottleneck, Stage III employs reinforcement learning to unlock true latent exploration. By adapting GRPO with our proposed novel latent policy optimization SGLP, we enable the model to actively explore the continuous visual manifold. 


In summary, our main contributions are:
\begin{itemize}[leftmargin=*]
\item We propose "Decompose, Look, and Reason", a reinforced latent reasoning framework, which dynamically decomposes the query into premises requiring visual verification, while simultaneously extracts premise-conditioned visual latents. The two components are mutually reinforcing, enabling the framework to progressively improve both the VLM text policy and latent visual policy to find an optimal reasoning trajectory.
\item We introduce a progressive three-stage training pipeline with a novel latent policy optimization. The proposed Spherical Gaussian Latent Policy (SGLP) intrinsically aligns with the hyperspherical manifold of vision-language representations and enables direct latent exploration without magnitude collapse. This effectively bridges the gap in multimodal latent RL and breaks the deterministic limitations of SFT.
\item We conduct extensive experiments on multiple vision-language benchmarks spanning visual perception, mathematical reasoning, and vision-language understanding. DLR consistently outperforms strong baselines, including text-only, interleaved MCoT, and latent reasoning methods. Further ablations and case studies verify the contributions of each component, while also demonstrating that DLR yields more interpretable stepwise visual reasoning through premise-conditioned latents.
\end{itemize}

\section{Related Work}
\subsection{Multimodal Chain-of-Thought Reasoning}
Multimodal Chain-of-Thought (MCoT) reasoning extends textual CoT to settings where the input, intermediate reasoning process, or output may involve non-linguistic modalities such as images~\cite{wang2025multimodal}. Early MCoT methods generate textual intermediate rationales conditioned on visual inputs. Representative works ~\cite{zhang2023multimodal,lu2022learn} show that inserting natural-language reasoning steps substantially improves visual question answering and science reasoning. But their intermediate reasoning process remains text-centric. Subsequent work explores interleaved MCoT, where visual content is explicitly inserted into the reasoning chain. For instance, Visual CoT introduces intermediate supervision by annotating a key bounding box~\cite{shao2024visual} while ICoT constructs interleaved reasoning by selecting image regions according to the model's attention.~\cite{gao2025interleaved}. These methods improve visual grounding, but still rely on patch-based visual embeddings.

\subsection{Think with Images}
Recent work has evolved from interleaved MCoT reasoning to "thinking with images"~\cite{su2025thinking}. In this paradigm, the model actively edits, manipulates, or augments visual evidence during intermediate reasoning steps. Representative methods equip vision-language models with drawing lines, or boxes during reasoning~\cite{hu2024visual}, and perform image edits such as highlighting, boxing, and masking to enhance structured visual understanding~\cite{fu2025refocus}. More recent methods further generalize this paradigm by learning to invoke external vision tools adaptively during reasoning~\cite{su2026pixel,su2025openthinkimg}. However, they typically rely on external tools or executable programs to manipulate images. As a result, they often incur additional computational overhead and remain constrained by the coverage and capability of the predefined toolset~\cite{su2025thinking}. In contrast, our method explores directly internal latent visual reasoning rather than external tool calls.

\subsection{Latent Space Reasoning}

An alternative line of work seeks to perform visual reasoning directly in continuous latent space. Representative methods differ in how these latents are defined and supervised. Some studies reconstruct visual embeddings by query-relevant ROI ~\cite{li2026latent,yang2025machine} and auxiliary images~\cite{wang2025monet,zhu2024explaining}. However, the explicit ROI supervision limits the expressiveness and scalability of the visual latents. The intermediate reasoning process in these methods is typically interleaved by a single visual embedding. More recently, IVT-LR~\cite{chen2026reasoning} can perform interleaved
vision-text reasoning in latent space by combining latent textual
states with selected visual embeddings. In contrast, our method could provide multi-step visual verifications without ROI supervision. It first establishes textual decomposition to iteratively decide \emph{what} and \emph{when} to inspect, then ground the specific premise with visual latents, and produce a more accurate rationale based on the retrieved evidence.

\section{Reinforced Latent Reasoning}
\begin{figure*}[t]
\begin{center}
\includegraphics[width=0.95\textwidth]{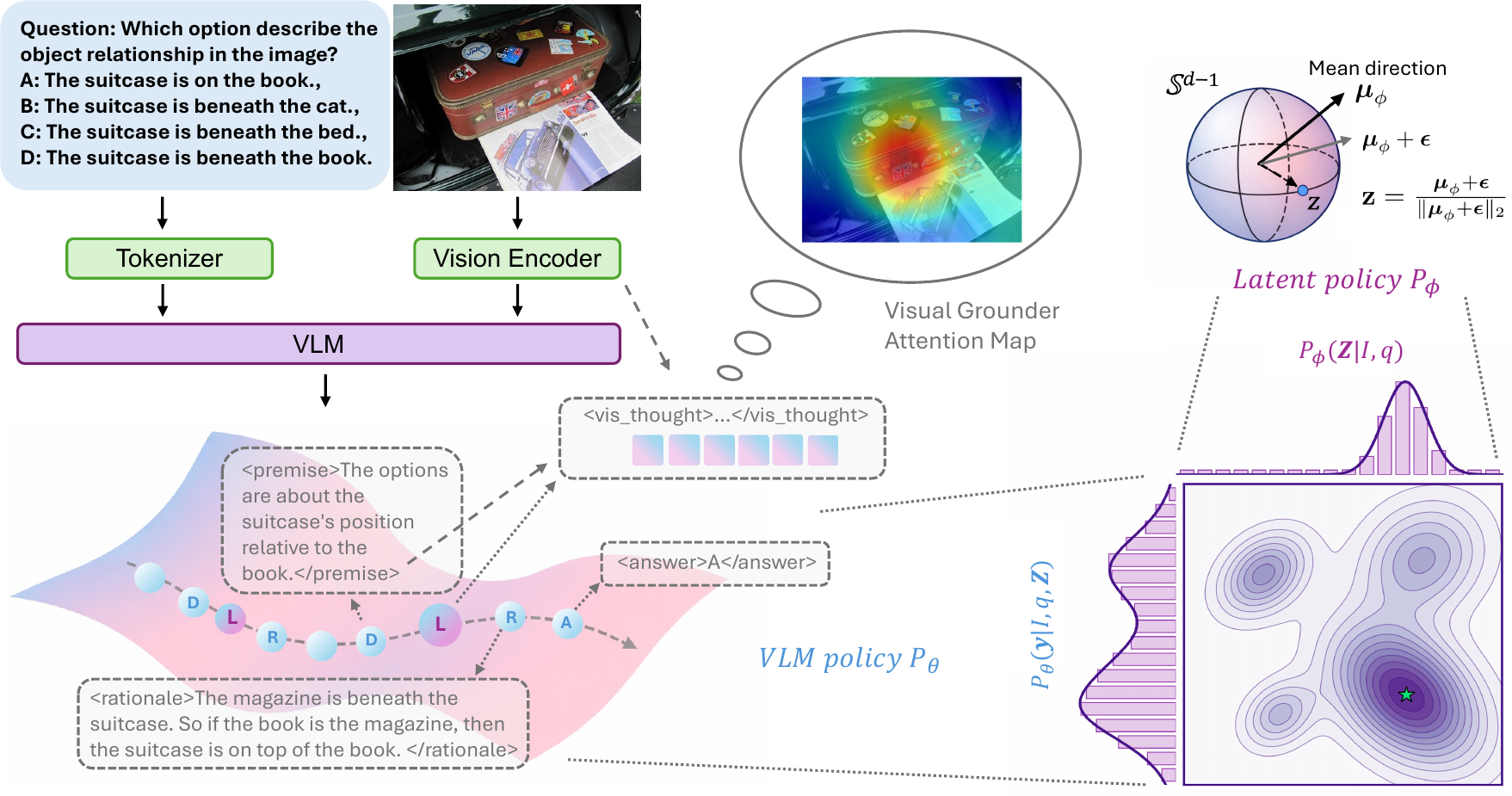}
\caption{Overview of the \emph{Decompose–Look–Reason} (DLR) framework and its reinforcement finetuning training objectives. The bottom-right contour illustrates the joint optimization landscape induced by the VLM text policy $p_{\theta}$ and the latent visual policy $p_{\phi}$.}
\label{fig:main}
\end{center}
\end{figure*}
\subsection{Method Overview}
We formulate the multimodal reasoning task as a multi-step reasoning process, following a "Decompose, Look, and Reason" paradigm. Given an input image $I$ and a user question $q$, our framework dynamically interleaves the visual grounding subquestion or premise $p$ with latent visual evidence $\mathbf{z}$ and its corresponding rationale $r$. This process generates an iterative reasoning trajectory $\tau = \{(p^{(t)}, \mathbf{z}^{(t)}, r^{(t)})\}_{t=1}^N$ that ultimately deduces the final answer $a$.

The architecture consists of two primary learnable components parameterized by $\Theta = \{\theta, \phi\}$: a VLM policy $P_\theta$ governing discrete text generation, and a latent visual grounder policy $P_\phi$ responsible for continuous latent embeddings. The generation process unfolds in three steps:

\begin{itemize}[leftmargin=*, label=$\diamond$]
    \item \textbf{Decompose:} The VLM $P_\theta$ learns to parse the current reasoning trajectory to mark a textual premise $p$ that needs further visual verification, enclosing it within the special tokens \texttt{<premise>} and \texttt{</premise>}.
    \item \textbf{Look:} The latent visual grounder $P_\phi$ takes the image features and the hidden state of \texttt{</premise>}, $s_t$, to actively search for visual evidence with respect to $p$, yielding continuous latent embeddings $\mathbf{z} \in \mathbb{R}^{l \times d}$ inside the \texttt{<vis\_thought>} and \texttt{</vis\_thought>} tokens. 
    \item \textbf{Reason:} Conditioned on the injected visual evidence $\mathbf{z}$, the VLM $P_\theta$ generates a textual rationale $r$ enclosing with the \texttt{<rationale>} and \texttt{</rationale>} and deduces the final answer $a$. 
\end{itemize}
As shown in Fig.~\ref{fig:main}, our framework couples textual reasoning and latent visual grounding through a joint factorization of the reasoning trajectory:
\begin{equation}
     P_{\Theta}(\mathbf{y},\mathbf{z} | I,q) = P_{\theta}(\mathbf{y}| I,q,\mathbf{z} ) P_{\phi}(\mathbf{z} | I,q),
\end{equation}
which can be further expanded in Eq.~\ref{eq:factorization}. This formulation makes the two components jointly optimized: \emph{better textual decomposition guides more precise visual grounding, and richer visual latents ground stronger rationales.} As a result, jointly optimizing the two components allows them to mutually reinforce each other, making it more likely to discover an optimal reasoning trajectory.

\subsection{Three-stage Training Pipeline}
To effectively train this dynamic and interleaved latent reasoning framework, we propose a three-stage training pipeline. 

\subsubsection{Stage I: Pretraining}

At the first stage, we warm up the latent visual grounder to establish basic cross-modal alignment between the continuous visual space and the discrete textual representation space of the VLM. 

\noindent \textbf{Visual Grounder} We freeze the pretrained VLM backbone and optimize a lightweight visual grounder. The frozen visual encoder extracts image features:
    $V=\Phi_{\mathrm{vis}}(I) \in \mathbb{R}^{m \times d},$
where $m$ is the number of visual tokens after the vision encoder and $d$ is the hidden dimension.

On the language side, we feed the question $q$ into the frozen language model and take the hidden state of its last valid token as the textual condition:
    $h_q=\Phi_{\mathrm{txt}}(q) \in \mathbb{R}^d.$

Similarly, we encode the answer $a$ with the same frozen language model and use the hidden state of its last valid token as the target text embedding:
    $h_a=\Phi_{\mathrm{txt}}(a) \in \mathbb{R}^d.$
    
We introduce a set of $L$ learnable latent queries
    $\mathbf{z_0} \in \mathbb{R}^{L \times d},$
which serve as trainable slots for extracting question-relevant visual evidence. 

The visual grounder contains two cross-attention layers followed by a feed-forward network. First, the latent queries are conditioned on $h_q$. Next, the question-conditioned latent queries attend over image features $V$ to extract relevant evidence. Finally, it passes a feed-forward network to obain the latent visual representation: $\mathbf{z} \in \mathbb{R}^{L \times d}$.

\noindent \textbf{Contrastive Alignment} To ensure the latent embeddings $\mathbf{z}$ capture the precise visual semantics necessary to deduce the answer, we employ an InfoNCE contrastive loss. We mean-pool the latent $\bar{\mathbf{z}} = \frac{1}{L} \sum_{k=1}^L \mathbf{z}_k$ and $L_2$-normalize it alongside the target answer embedding:
\begin{equation}
    \hat{\mathbf{z}}=\frac{\bar{\mathbf{z}}}{\|\bar{\mathbf{z}}\|_2},\quad\hat{\mathbf{h}}_a=\frac{\mathbf{h}_a}{\|\mathbf{h}_a\|_2}.
\end{equation}
The bidirectional contrastive loss guides the pooled visual evidence to align with the correct answer's semantic embedding in the latent space:
\begin{equation*}
    \mathcal{L}_{v \to t}=-\frac{1}{N}\sum_{i=1}^N\log\frac{\exp(\hat{\mathbf{z}}_i^\top\hat{\mathbf{h}}_{a,i}/\tau)}{\sum_{j=1}^N\exp(\hat{\mathbf{z}}_i^\top\hat{\mathbf{h}}_{a,j}/\tau)},
\end{equation*}
\begin{equation*}
     \mathcal{L}_{t \to v}=-\frac{1}{N}\sum_{i=1}^N\log\frac{\exp(\hat{\mathbf{h}}_{a,i}^\top\hat{\mathbf{z}}_i/\tau)}{\sum_{j=1}^N\exp(\hat{\mathbf{h}}_{a,j}^\top\hat{\mathbf{z}}_i/\tau)},
\end{equation*}
\begin{equation}
    \mathcal{L}_{pretrain}=\frac{1}{2}(\mathcal{L}_{v \to t}+\mathcal{L}_{t \to v}).
\end{equation}
\subsubsection{Stage II: Supervised Finetuning (SFT)}
In the second stage, we teach the model to follow our "Decompose, Look, and Reason" paradigm through supervised finetuning. We construct an annotated visual reasoning SFT dataset based on Vision-R1-cold dataset~\cite{huang2025vision} using an LLM (e.g., GPT-5-mini) with the prompt in Appendix~\ref{app:sft_prompt} to add our DLR special tokens. During SFT, both the VLM $\pi_\theta$ and the visual grounder $\pi_\phi$ are jointly optimized. The VLM learns to break down the reasoning trajectory into a sequence of image related sub-questions or premises $p$. Simultaneously, the visual grounder is trained as a deterministic feature extractor to output latent embeddings $\mathbf{z}^{(t)}$ that maximize the likelihood of the ground-truth rationales and the final answer.  The joint probability distribution for generating this trajectory can be factorized into the product of the VLM policy $\pi_\theta$ and the latent visual policy $\pi_\phi$:
\begin{equation}
\label{eq:factorization}
    \small
    \begin{aligned}
        P_{\Theta}&(\mathbf{y},\mathbf{z} | I,q) = \prod_{t=1}^N \Big[ \underbrace{P_\theta(p^{(t)} | I, q, \mathcal{H}_{<t})}_{\text{Decompose}} \cdot \underbrace{P_\phi(\mathbf{z}^{(t)} | I, s_t)}_{\text{Look}} \\
        & \cdot \underbrace{P_\theta(r^{(t)} | I, q, \mathcal{H}_{<t}, p^{(t)}, \mathbf{z}^{(t)})}_{\text{Reason}} \Big] \cdot P_\theta(a | I, q, \mathcal{H}_{\le N}),
    \end{aligned}
\end{equation}
where $\mathcal{H}_{<t}$ denotes the cumulative discrete context history up to step $t$. 

In practice, optimizing the continuous latent variables $\mathbf{z}^{(t)}$ with a discrete cross-entropy objective is infeasible. Therefore, we represent the latents with a sequence of $L$ discrete placeholder tokens (e.g., \texttt{<VIS0>} to \texttt{<VIS31>}) and implement a two-pass forward mechanism. Specifically, the VLM first extracts the hidden state $s_t$ at the \texttt{</premise>} token, which the grounder $\pi_\phi$ then uses to generate the continuous visual latents $\mathbf{z}^{(t)}$. Next, we replace the placeholder tokens with $\mathbf{z}^{(t)}$ and continue generating the subsequent rationale $r^{(t)}$ and answer $a$. When computing the training objective, we mask the labels of these placeholder tokens. The optimization objective is thus formulated as the autoregressive Cross-Entropy (CE) loss over the text tokens:
\begin{equation}
    \small
    \begin{aligned}
        \mathcal{L}&_{SFT}(\Theta) =  - \mathbb{E}_{\mathbf{y} \sim \mathcal{D}} \Bigg[\sum_{t=1}^N \Big( \log \pi_\theta(p^{(t)} | I, q, \mathcal{H}_{<t}) + \\ & \log \pi_\theta(r^{(t)} | I, q, \mathcal{H}_{<t}, p^{(t)}, \mathbf{z}^{(t)}) \Big) 
        + \log \pi_\theta(a | I, q, \mathcal{H}_{\le N}) \Bigg].
    \end{aligned}
\end{equation}
By minimizing this CE loss, the visual grounder $\pi_\phi$ is implicitly optimized via gradients backpropagated from the subsequent rationales $r^{(t)}$ and the final answer $a$. While SFT successfully teaches the model to internalize the structured DLR format, it relies strictly on teacher-forced log-likelihood. This inherently bounds the visual grounder's capacity to deterministic feature extraction without active exploration, which motivates our latent policy in Stage III.

\subsubsection{Stage III:Reinforcement Finetuning for Continuous Latent Embeddings}

\noindent \textbf{Limitations of GRPO and Geometric Inductive Bias on Latent Embeddings.}
Prior work on latent visual reasoning~\cite{li2026latent,yang2025machine} perform GRPO~\cite{shao2024deepseekmath} following SFT. But standard GRPO can only be applied to discrete text tokens and can not optimize continuous latent embeddings due to the lack of a defined token distribution. Thus, previous work treats visual latents as deterministic variables, confining the RL exploration to only the LLM's text generation. The visual latents are frozen or updated indirectly via gradients backpropagated from the textual RL loss. This lack of a visual latent direct exploration limits the latent representation learning, easily falling into the suboptimal bounded by its SFT.

To address this limitation, we propose a stochastic latent policy for the visual grounder. Crucially, this policy should be aligned with the geometry of the vision-language feature space: it is inherently constrained to a hyperspherical manifold, where semantic information (e.g., cosine similarity in CLIP) is encoded primarily in the \emph{direction} rather than the \emph{magnitude} of feature vectors. Formally, let the semantic manifold be the unit hypersphere $\mathbb{S}^{d-1} = \{ \mathbf{x} \in \mathbb{R}^d : \|\mathbf{x}\|_2 = 1 \}$. A standard Gaussian distribution $\mathcal{N}(\boldsymbol{\mu}, \sigma^2 \mathbf{I})$ defined in Euclidean space $\mathbb{R}^d$ is geometrically mismatched for this manifold, as it conflates magnitude variance with semantic variance~\cite{wang2025monet}. To resolve this, we propose a Spherical Gaussian Latent Policy (SGLP), which intrinsically aligns with the geometric inductive bias of the vision-language feature space. Instead of applying unconstrained Euclidean noise, our visual grounder predicts a $L_2$-normalized mean direction $\boldsymbol{\mu}_\phi \in \mathbb{S}^{d-1}$. We then inject isotropic exploration noise and explicitly re-project the sampled vector back onto the unit hypersphere:
\begin{equation}
\mathbf{z} = \frac{\boldsymbol{\mu}_\phi + \boldsymbol{\epsilon}}{\| \boldsymbol{\mu}_\phi + \boldsymbol{\epsilon} \|_2}, \quad \text{where } \boldsymbol{\epsilon} \sim \mathcal{N}(\mathbf{0}, \sigma^2 \mathbf{I}).
\end{equation}
This formulation ensures that the continuous exploration operates on the angular space, since its norm is fixed as $\|\mathbf{z}\|_2=1$. By mathematically decoupling the semantic direction from the vector magnitude, our approach eliminates the risk of magnitude collapse. 

\noindent \textbf{Latent Policy Optimization.}
We adapt the Dr. GRPO~\cite{liu2025understanding} for our joint policy optimization. For each question $q$, we sample a group of $G$ trajectories $\{\tau_1, \dots, \tau_G\}$ from the current joint policy $\pi_{\Theta_{old}}$, where $\Theta = \{\theta, \phi\}$. The advantage $A_i$ for the $i$-th trajectory is estimated by subtracting the group mean:
\begin{equation}
    A_i = r^{(i)} - \mu(\{r^{(j)}\}_{j=1}^G).
\end{equation}
The total objective $\mathcal{J}(\Theta)$ is the sum of the text policy objective and the latent policy objective:
\begin{equation}
    \mathcal{J}(\Theta) = \mathcal{J}_{text}(\theta) + \mathcal{J}_{latent}(\phi).
\end{equation}
While $\mathcal{J}_{text}(\theta)$ follows Dr.~GRPO formulation for discrete text tokens, we derive $\mathcal{J}_{latent}(\phi)$ for our latent policy. For a latent $\mathbf{z}_{old}$ generated by the old policy mean $\boldsymbol{\mu}_{\phi_{old}}$, the optimization objective is:
\begin{equation}
\small
\mathcal{J}_{latent}(\phi) = \mathbb{E} \left[ \frac{1}{G} \sum_{i=1}^G \min \left( \rho_i A_i, \text{clip}(\rho_i, 1-\epsilon, 1+\epsilon) A_i \right) \right],
\end{equation}
where the importance sampling ratio $\rho_i$ is computed using the spherical Gaussian probability densities:
\begin{equation}
\small
\begin{aligned}
\rho_i &= \frac{\pi_\phi(\mathbf{z}_{old}^{(i)})}{\pi_{\phi_{old}}(\mathbf{z}_{old}^{(i)})}  \\
&= \exp\left( \frac{\|\mathbf{z}_{old}^{(i)} - \boldsymbol{\mu}_{\phi_{old}}^{(i)}\|_2^2 - \|\mathbf{z}_{old}^{(i)} - \boldsymbol{\mu}_{\phi}^{(i)}\|_2^2}{2\sigma^2} \right).
\end{aligned}
\end{equation}
\begin{theorem}[Latent Policy Effect]
By maximizing $\mathcal{J}_{latent}(\phi)$, the expected update of the visual grounder increases the alignment between mean direction $\boldsymbol{\mu}_\phi$ and latent samples $\mathbf{z}_{old}$ that receive high advantage. Equivalently, the optimization encourages the latent policy to move toward semantically task-relevant directions on the hypersphere.
\end{theorem}
\begin{proof} 
The formal proof is in Appendix~\ref{app:gradient}. 
\end{proof}

\noindent \textbf{Reward Design.} We design a dense reward function $r$ to guide the exploration. The total reward for the $i$-th trajectory is defined as:
\begin{equation}
    r^{(i)} = R_{outcome}^{(i)} + \beta \cdot \mathbb{I}\left(R_{outcome}^{(i)} > 0\right) \cdot R_{focus}^{(i)}.
\end{equation}

\noindent \textbf{(1) Outcome Reward ($R_{outcome}$)}: This is a binary sparse reward indicating the correctness of the model-generated answer $\hat{a}$ against the ground-truth answer $a$.
\begin{equation}
R_{outcome}^{(i)} = \begin{cases} 
1 & \text{if } \text{ExactMatch}(\hat{a}^{(i)}, a^{(i)}) \\
0 & \text{otherwise}
\end{cases}
\end{equation}
\noindent \textbf{(2) Visual Grounder's Focus Reward ($R_{focus}$)}: To encourage visual grounding, we align the visual grounder's attention map $\mathcal{A}_{latent}$ with an oracle attention map $\mathcal{A}_{oracle}$ derived from a frozen, strong vision-language model (e.g., SigLIP-SO400M) conditioned on the premise $p$. We define the focus reward as the exponential negative Kullback-Leibler (KL) divergence:
\begin{equation}
    R_{focus}^{(i)} = \exp\left( -\lambda \cdot D_{KL}(\mathcal{A}_{oracle}^{(i)} \| \mathcal{A}_{latent}^{(i)}) \right).
\end{equation}
The indicator function $\mathbb{I}(\cdot)$ in the total reward acts as a filter: we postulate that visual alignment is only meaningful when it contributes to a correct inference. If the model answers incorrectly, the focus reward is suppressed to zero to avoid reinforcing "hallucinated" attention patterns that do not yield factual correctness.


\section{Experiments}

\subsection{Experiment Setup}
\noindent \textbf{Evaluation Benchmarks}
To assess our DLR framework, we conduct comprehensive evaluations across four vision-centric benchmarks that collectively measure diverse capabilities in visual reasoning and understanding. For visual detail understanding, we evaluate on V* Bench~\cite{wu2024v}, which assesses VLM's capability to perform fine-grained attribute recognition and relative spatial reasoning. For mathematical visual reasoning, we evaluate on MathVista~\cite{lu2024mathvista}, a benchmark designed to assess mathematical reasoning in visual contexts. For broad multidisciplinary multimodal reasoning, we evaluate on MMMU-Pro~\cite{yue2025mmmu}. Finally, for general multimodal capability, we evaluate on MMStar~\cite{chen2024we}, a vision-indispensable benchmark that covers six core capabilities and eighteen fine-grained axes. We report the Pass@1 accuracy using greedy decoding. We set a maximum length of 40960.

\noindent \textbf{Baselines}
We compare DLR against several representative vision-language reasoning baselines. These include our model backbone Qwen3-VL-8B-Thinking~\cite{bai2025qwen3vltechnicalreport} which is a strong text-only baseline; ICoT~\cite{gao2025interleaved}, an interleaved multimodal CoT baseline that constructs paired visual-textual intermediate rationales by selecting image patches based on attention maps; PixelReasoner~\cite{su2026pixel}, a "thinking with images" approach that performs intermediate reasoning through pixel-space visual operations; and LVR~\cite{li2026latent}, a latent visual reasoning method that conducts reasoning in continuous visual embedding space. All open-source baselines are adapted on top of the same Qwen3-VL-8B-Thinking backbone as DLR to ensure fair comparison.

\noindent \textbf{Implementation Details}
We adopt Qwen3-VL-8B-Thinking as our base VLM. The visual grounder is initialized with $L=32$ learnable continuous queries. The focus reward parameter $\beta=0.1$ and $\lambda=1$. Experiments are conducted on 4 NVIDIA H200 GPUs. More details in Appendix~\ref{app:implementation}.

\subsection{Main Results}

\begin{table*}[htbp]
\centering
\resizebox{0.8\textwidth}{!}{
\begin{tabular}{l ccc c c c}
\toprule

\multirow{2}{*}{\textbf{Model}} & \multicolumn{3}{c}{\textbf{V*}} & \textbf{MathVista} & \textbf{MMMU-Pro} & \textbf{MMStar} \\
\cmidrule(lr){2-4} \cmidrule(lr){5-5} \cmidrule(lr){6-6} \cmidrule(lr){7-7}
& Overall & Attribute & Spatial & Testmini & Overall & Overall \\
\midrule

\multicolumn{7}{c}{\textit{Proprietary Model}} \\
\midrule
GPT-4o & 67.5* & 72.2* & 60.5* & 60.0* & 51.9* & - \\
\midrule

\multicolumn{7}{c}{\textit{Open-Source Model based on Qwen3-VL-8B-Thinking}} \\
\midrule
Qwen3-VL-8B-Thinking & 79.6 & 81.7 & 76.3 & 79.2 & 60.9 & 72.3 \\
ICoT& 79.6 & 81.7 & 76.3  &77.1 & 61.2 & 67.0 \\
LVR & 82.2 & 82.6 & 81.6 & 77.3 & 60.5 & 72.7 \\
PixelReasoner & 80.1 & 82.6 & 76.3& 74.2 & 56.1 & 68.7 \\
\midrule

\rowcolor[HTML]{E6F0FA} 
DLR (ours) & \textbf{83.8} & \textbf{84.3} & \textbf{82.9} & \textbf{82.7} & \textbf{63.5} & \textbf{76.2} \\
\rowcolor[HTML]{E6F0FA} 
 \textit{$\Delta$ over Backbone} & \textcolor[HTML]{00A000}{+4.2} & \textcolor[HTML]{00A000}{+2.6} & \textcolor[HTML]{00A000}{+6.6} & \textcolor[HTML]{00A000}{+3.5} & \textcolor[HTML]{00A000}{+2.6} & \textcolor[HTML]{00A000}{+3.9} \\

\bottomrule
\end{tabular}
}
\caption{Performance on real-world perception, math, and reasoning benchmarks. The best results are highlighted in \textbf{bold}. Results marked with ``*'' are reported by other papers~\cite{wang2025vl,wang2025monet}.}
\label{tab:performance}
\end{table*}

Tab.~\ref{tab:performance} presents the main results of DLR on four vision-centric benchmarks. Overall, DLR achieves the best performance among all baselines across all evaluated benchmarks, even surpassing the proprietary model GPT-4o with approximately 200B parameters. On V* Bench, DLR achieves 83.8 overall accuracy, outperforming the backbone by 4.2\%, LVR by 1.6\%, and PixelReasoner by 3.7\%. The gains are consistent across both attribute and spatial subcategories, confirming that DLR improves fine-grained visual detail understanding in both attribute recognition and relative position reasoning. We adapted baselines to the same Qwen3-VL-8B-Thinking backbone as ours for a fair comparison. ICoT inserts its interleaved visual evidence based on "\texttt{\textbackslash n}" tokens. On short reasoning tasks like V*, the model would fail to trigger newlines, causing ICoT to degenerate into the backbone performance. LVR's design is inherently compatible with Qwen3-VL-Thinking's \texttt{<think>} block: its latent visual reasoning tokens (\texttt{<|lvr\_start|>}, \texttt{<|lvr\_end|>}) are inserted inside the \texttt{<think>} block, and its SFT data preserves the CoT structure. As a result, LVR benefits from the stronger Qwen3-VL backbone. In contrast, PixelReasoner's tool calls are placed outside of the \texttt{<think>} block in Qwen3-VL's implementations by design, so PixelReasoner's SFT data cannot be wrapped into the thinking mode. When PixelReasoner's SFT is applied with this non-thinking data format for 1–2 epochs, the model tends to prioritize learning the tool-calling pattern at the expense of forgetting the original thinking capability that the backbone was trained on, which prevents it from building upon the backbone's existing reasoning abilities.  On MathVista, DLR obtains 82.7, exceeding Qwen3-VL baseline by 3.5\% and surpassing the strongest open-source baseline LVR by 5.4\%. This result suggests that the proposed DLR mechanism is particularly beneficial for mathematical visual reasoning, where the model must progressively inspect diagrams or figures and accumulate evidence over multiple steps rather than relying on a single-pass latent representation like LVR. On MMMU-Pro, DLR achieves improves over Qwen3-VL baseline by 2.6\%. This shows that DLR generalizes beyond localized perception tasks to challenging multidisciplinary reasoning scenarios, where better premise decomposition and premise-conditioned visual verification both contribute to improved answer accuracy. Results on MMStar show that DLR improves not only improves the complex reasoning tasks but also broader multimodal understanding.

In summary, we have several key findings: (i) \textbf{the text-only baseline is consistently weaker than methods that explicitly incorporate visual evidence in intermediate steps}, confirming the importance of grounded visual verification. In complex tasks, Qwen3-VL baseline sometimes generates very long outputs (over 30k+ tokens) but repeats uncertain self-verification without sufficient visual grounding. (ii) \textbf{"think with images" methods such as PixelReasoner improve over the text-only baseline, but stills trails latent visual reasoning methods like LVR.} It suggests that internal latent grounding can be more efficient than external visual-editing tools. (iii) \textbf{among latent reasoning methods, DLR consistently outperforms LVR}, indicating \textbf{dynamically interleaving multi-step premise-conditioned latent visual grounding is more effective} than constraining the reasoning process to a single, coarse latent embedding.

\subsection{Ablation Study}
\begin{table}[t]
\centering
\resizebox{\linewidth}{!}{
\begin{tabular}{l c c c c}
\toprule
\textbf{Model} & \textbf{V*} & \textbf{MathVista} & \textbf{MMMU-Pro} & \textbf{MMStar} \\
\midrule
DLR-SFT & 80.5 & 65.1 & 54.6 & 64.6 \\
\quad w/o pretraining & 79.5 & 63.0 & 53.6 & 63.2 \\
\midrule
\rowcolor[HTML]{EBF1F8} 
DLR & \textbf{83.8} & \textbf{82.7} & \textbf{63.5} & \textbf{76.2} \\
\quad w/o $R_{focus}$ & 82.6 & 66.5 & 55.1 & 64.8 \\
\quad w/o $\mathcal{J}_{latent}$ & 81.5 & 57.1 & 54.9 & 63.3 \\
\bottomrule
\end{tabular}
}
\caption{Ablations of the components in DLR.}
\label{tab:ablation}
\end{table}

\begin{figure*}[t]
\begin{center}
\includegraphics[width=\textwidth]{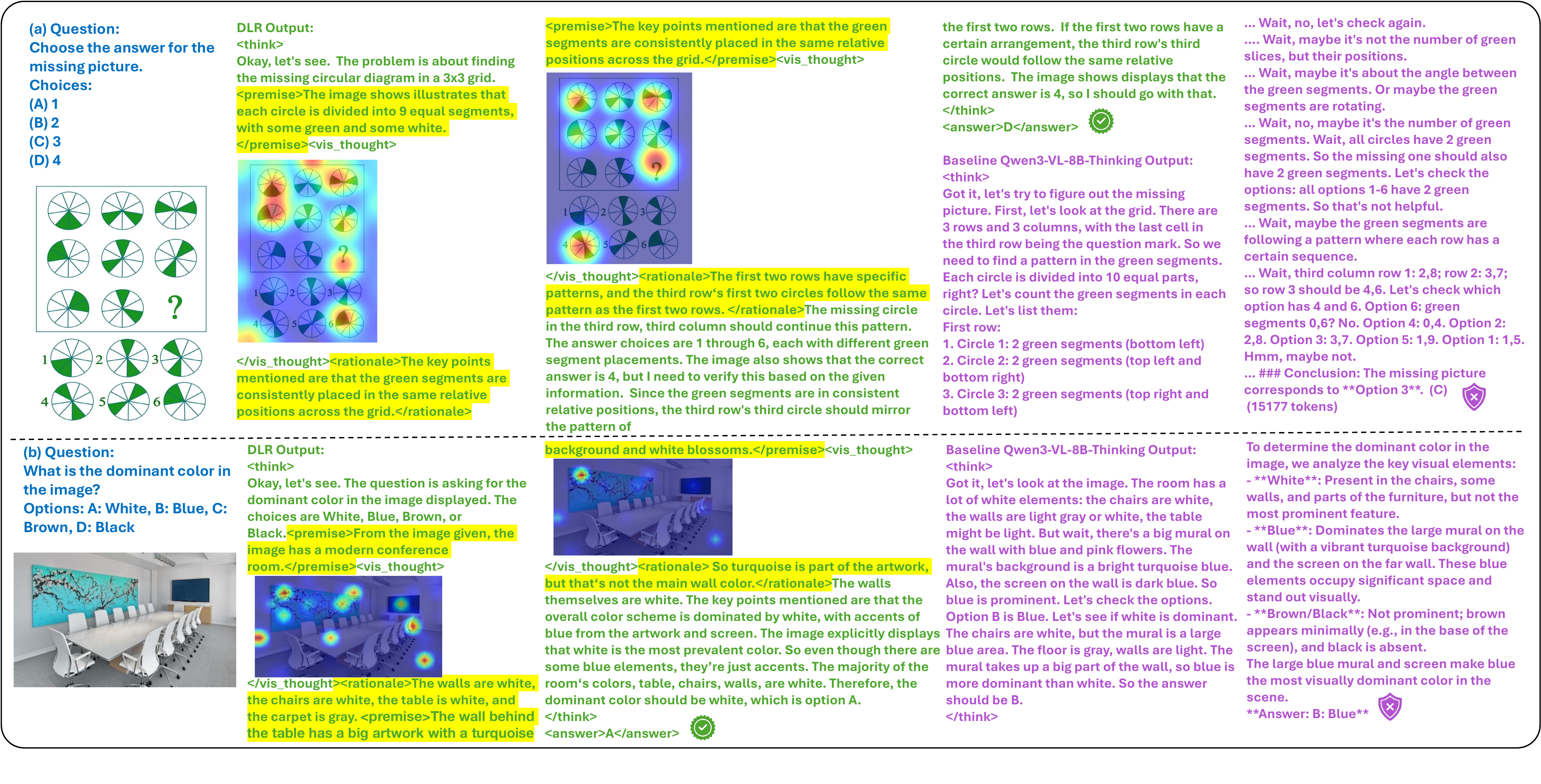}
\caption{Case studies between DLR and baseline Qwen3-VL-8B-Thinking. The premises and rationales in DLR are highlighted in yellow. These examples show that precise decomposed premises and visually grounded rationales lead to correct answers where the baseline fails.}
\label{fig:qualitative}
\end{center}
\end{figure*}
\noindent \textbf{Effect of the Pretraining, SFT, and RL.} In Tab.~\ref{tab:ablation}, we observe that removing pretraining consistently hurts performance on all benchmarks. This shows that the pretraining stage provides a better initialization for the visual grounder by establishing foundational cross-modal alignment before structured reasoning is learned during SFT. Without this warm-up stage, the latent visual grounder is less capable of extracting premise-relevant visual evidence. During SFT, it internalizes the structured DLR format. Compared with DLR-SFT, the full DLR achieves a 3.3\% improvement on V*. This shows that the RL stage can break the deterministic limitations of SFT, enabling the model to actively explore visual evidence and to facilitate finding potentially correct reasoning paths.

\noindent \textbf{Effect of the Latent Policy Optimization.} The most critical component of stage III is the latent policy optimization. Removing $\mathcal{J}_{latent}$ in Tab.~\ref{tab:ablation} leads to a drastic performance degradation, particularly on MathVista, where the score drops from 82.7\% to 57.1\%. This confirms that our proposed SGLP is essential for guiding effective exploration in continuous visual space.

\noindent \textbf{Effect of the Focus Reward.}
Removing $R_{focus}$ results in a decline in V* performance by 1.2\%, suggesting that aligning the visual grounder's attention with an oracle grounding signal provides some guidance for stable latent exploration.



\subsection{Quantitative Faithfulness Analysis}

Beyond the qualitative attention visualizations, we further 
quantitatively evaluate whether the visual grounder's attention 
faithfully identifies image regions that are important to the 
model's predictions. Following the comprehensiveness framework~\cite{deyoung2020eraser}, we mask the top image regions 
identified by the visual grounder's attention map and measure the 
resulting accuracy degradation on V* Bench. As a control, we 
randomly mask the same number of image regions.

\begin{table}[t]
\centering
\begin{tabular}{lcc}
\toprule
\textbf{Masking Strategy} & \textbf{V* Accuracy} & \textbf{$\Delta$} \\
\midrule
No mask (DLR)       & 83.8 & --    \\
Random mask         & 77.7 & -6.1  \\
DLR attention mask  & 36.5 & -47.3 \\
\bottomrule
\end{tabular}
\caption{Quantitative faithfulness analysis on V* Bench.}
\label{tab:masking_analysis}
\end{table}

As shown in Tab.~\ref{tab:masking_analysis}, random masking reduces accuracy by only 6.1 percentage points, whereas masking DLR attention regions causes a 47.3-point drop. This large degradation confirms that the visual grounder's attention faithfully captures the image regions critical to the model's predictions.

\subsection{Case Study}
 Both cases in Fig.~\ref{fig:qualitative} highlight \textbf{a common challenge for text-only multimodal reasoning: \emph{The model often generates long but weakly visually grounded reasoning chains.}} In contrast, DLR explicitly decomposes the problem into premise-level subprobelms, retrieves premise-conditioned visual evidence, and then produces grounded rationales before answering. In the first example, the baseline severely \textbf{overthinks} the problem. Although it repeatedly waits to check possible rules for the circular diagrams again and again, it fails to reliably ground these hypotheses in the image and ultimately predicts the wrong answer after generating \textbf{15,177 tokens}. The baseline keeps revising speculative rules such as counting green segments, checking angular relations, and testing row-wise patterns, but still converges to an incorrect option. This example reveals that a text-only reasoning process can become verbose and unstable when the task fundamentally depends on precise visual comparison. By contrast, DLR solves the same problem through a much more structured reasoning process. It first generates a premise that identifies the key visual subproblem---namely, that the green segments are arranged in consistent relative positions across the grid---and then retrieves visual evidence conditioned on this premise. The subsequent rationale is directly grounded in the retrieved evidence, allowing the model to verify the pattern step by step. Rather than forcing the model to solve the entire problem in a single opaque trajectory, DLR follows an explicit \emph{decompose $\rightarrow$ look $\rightarrow$ reason} loop that is better aligned with how humans solve difficult visual tasks. 

 The second example further demonstrates the interpretability of our framework. The text-only baseline overemphasizes the salient mural and incorrectly predicts \emph{blue}, whereas DLR correctly infers that \emph{white} is the dominant color because most of the room elements---including the walls, chairs, and table---are white. Meanwhile, the attention heatmaps in Fig.~\ref{fig:qualitative} are derived from the image cross-attention layer within the visual grounder module. In both examples, the highlighted regions align well with the image areas that are relevant to the current subproblem, indicating that the retrieved latent evidence is tightly guided by the textual premise. This provides a level of \textbf{stepwise interpretability} that is absent in current latent visual reasoning methods.

\section{Conclusion}
We present DLR, a reinforced latent reasoning framework that unifies dynamic textual decomposition and premise-conditioned visual grounding for multi-step multimodal reasoning, making them jointly optimized and mutually beneficial. Extensive experiments demonstrate that DLR consistently improves performance across all benchmarks, while provides stepwise interpretability.

\section*{Limitations}

Our evaluation mainly focuses on vision-centric (i.e., image-centric) reasoning benchmarks. The generality of the framework to broader multimodal settings—such as video reasoning, embodied multimodal decision-making, or tasks requiring richer external interaction—remains to be validated. Extending this work to these more challenging settings is an important direction for our future work.

Another limitation concerns the transparency--efficiency trade-off of
latent reasoning. Compared with explicit tool-based methods, whose
intermediate operations such as crops, zooms, or sketches are directly
inspectable, continuous latent representations are inherently less
human-readable. DLR partially mitigates
this limitation by retaining explicit premises and rationales around each
latent visual reasoning step and exposing premise-conditioned attention
maps. Therefore, developing a more interpretable mechanism for continuous visual thoughts is another important direction for future work.

\section*{Acknowledgments}
This work was supported by the National Science Foundation (NSF) under Grant Nos. 2432418,
2414115, 2403312, 2318831, and 2521631; the National Institutes of Health (NIH) under Grant
Nos. R01AG089806 and R01CA297856.

\bibliography{custom}

\clearpage
\appendix

\section{Prompt for Constructing SFT Data}
\label{app:sft_prompt}

\begin{figure}[htbp]
\begin{tcolorbox}[
  colback=gray!12,
  colframe=gray!60,
  boxrule=0.6pt
]

<image>\textbackslash nThe question is “\hl{\{question\}}” and the reasoning process is “\hl{\{reasoning\}}”. Within the reasoning process, focus on the content between \taghl{<think>} and \taghl{</think>} and perform the following annotation: Identify every major argument that requires looking at the image and is immediately followed by sentence(s) that describe the specific details obtained from observing the image, which serve as the crucial rationale for deriving the answer.

\bigskip

Wrap each such major argument with the tags \taghl{<premise>} and \taghl{</premise>}. Only tag the arguments that require looking at the image to get the following details. Wrap the corresponding supporting details with \taghl{<rationale>} and \taghl{</rationale>}. For example: <premise>The image shows white suitcases stacked.</premise><rationale>On the top suitcase in the foreground, there's a wide-brimmed hat with a black band and a red fabric draped over it.</rationale> or <premise>There are two parallel lines, m and n, cut by a transversal.</premise><rationale>Angle 1 is above line m on the right side of the transversal, and angle 2 is below line n on the left side of the transversal.</rationale>

\bigskip
The final output should be the original reasoning with \taghl{<premise>} and \taghl{<rationale>} tags inserted in place. Keep all original text unchanged except for inserting the tags.

\end{tcolorbox}
\caption{The Prompt Template for Constructing the Reasoning Process for SFT.}
\label{fig:query_prompt}
\end{figure}

\section{Gradient Derivation for Latent Policy Optimization}
\label{app:gradient}
We formally derive the gradient update rule for the latent visual grounder, parameterized by $\phi$, to demonstrate how the Spherical Gaussian Latent Policy (SGLP) explores the semantic space.

Let $\mathbf{z}^{(t)} \in \mathbb{S}^{d-1}$ be the continuous latent embedding generated at step $t$ conditioned on state $s_t$. From the Policy Gradient Theorem, the gradient of the expected return objective $\mathcal{J}_{latent}(\phi)$ with respect to the visual grounder parameters $\phi$ is given by:
\begin{equation*}
    \nabla_\phi \mathcal{J}_{latent}(\phi) = \mathbb{E} \left[ \sum_{t} A \nabla_\phi \log \pi_\phi(\mathbf{z}^{(t)} | s_t) \right],
\end{equation*}
For our SGLP, both the sampled latent $\mathbf{z}^{(t)}$ and the predicted mean direction $\boldsymbol{\mu}_\phi(s_t)$ are strictly $L_2$-normalized. As a result, the squared Euclidean distance mathematically simplifies into a scaled cosine similarity. The log-probability of the policy is therefore proportional to the dot product:
$$ \log \pi_\phi(\mathbf{z}^{(t)} | s_t) \approx \frac{1}{\sigma^2} {\mathbf{z}^{(t)}}^\top \boldsymbol{\mu}_\phi(s_t) + C. $$
Taking the derivative of this log-probability with respect to $\phi$ yields:
\begin{equation*}
\begin{aligned}
 \nabla_\phi \log \pi_\phi(\mathbf{z}^{(t)} | s_t) & \approx \frac{1}{\sigma^2} \nabla_\phi \left( {\mathbf{z}^{(t)}}^\top \boldsymbol{\mu}_\phi(s_t) \right) \\ &= \frac{1}{\sigma^2} {\mathbf{z}^{(t)}}^\top \nabla_\phi \boldsymbol{\mu}_\phi(s_t). 
 \end{aligned}
 \end{equation*}
Substituting this back into the policy gradient formulation, the parameter update rule governed by the GRPO advantage $A$ becomes:
$$ \phi_{new} \leftarrow \phi_{old} + \alpha \mathbb{E} \left[ \sum_{t} \frac{A}{\sigma^2} {\mathbf{z}^{(t)}}^\top \nabla_\phi \boldsymbol{\mu}_\phi(s_t) \right]. $$

Some Takeaways on the Manifold:
\begin{itemize}
    \item Directional Alignment (${\mathbf{z}^{(t)}}^\top \nabla_\phi \boldsymbol{\mu}_\phi$): Because both vectors reside on the unit hypersphere, the gradient guides the predicted mean vector $\boldsymbol{\mu}_\phi$ along the surface of the manifold without altering its magnitude.
    \item Advantage Weighting ($A$): The group-relative advantage $A$ acts as a directional scalar. If $A > 0$ (the better latent $\mathbf{z}^{(t)}$ focused on the right visual region and represents better extracted premise-aligned representation, which led to a correct rationale and answer), the gradient pulls the mean direction $\boldsymbol{\mu}_\phi$ towards the explored $\mathbf{z}^{(t)}$. Conversely, if $A < 0$, it pushes $\boldsymbol{\mu}_\phi$ away.
\end{itemize}

\section{Implementation Details}
\label{app:implementation}
All trainings use the AdamW optimizer with a weight decay of 0.01 and a cosine learning rate scheduler with a 10\% linear warmup.

\textbf{Stage I: Pretraining.} In the first stage, the base VLM is completely frozen. We train only the visual grounder on 443,757 visual question-answering pairs in the training set of VQAv2 dataset~\cite{goyal2017making} using the symmetric InfoNCE loss. We set the contrastive temperature $\tau=0.07$ and train 3 epochs with a learning rate of $1 \times 10^{-4}$ and batch size of 16.

\textbf{Stage II: SFT.} We train the model on our custom annotated Vision-R1-cold dataset (200k)~\cite{huang2025vision} to internalize the DLR format. During SFT, the visual encoder and the base LLM weights remain frozen. We apply LoRA fine-tuning with rank $r = 64$, $\alpha = 128$, and a dropout rate of $0.05$. We train for 2 epochs with a learning rate of $5 \times 10^{-5}$ and an effective batch size of 16. 

\textbf{Stage III: Latent Policy Optimization.} For the reinforcement finetuning stage, we implement our SGLP framework on the ViRL dataset (39k)~\cite{wang2025vl}. The policy samples a group of $G=4$ trajectories per query. We train for 1 epoch with a VLM learning rate is $1 \times 10^{-6}$ with sampling temperature set to $1.0$, and an effective batch size of 8. The visual grounder learning rate is $1 \times 10^{-5}$.

\section{Computational Overhead.}
DLR introduces only a lightweight visual grounder of approximately
7M parameters, corresponding to less than 0.09\% of the 8B VLM
backbone. During Stages II and III, the grounder is jointly optimized
with the VLM and therefore introduces only a small additional
parameter overhead. The only additional training stage compared with
a standard SFT + RL pipeline is Stage I pretraining, where the entire
VLM backbone is frozen and only the visual grounder is optimized.
Moreover, the relative parameter overhead decreases with backbone
scale, accounting for less than 0.01\% when scaling to 70B+ models.

\section{LLM Usage Disclosure}
We use large language models to correct the grammar and improve the clarity of writing in this paper.
\end{document}